\documentclass{article} %

    \usepackage{arxiv}

\usepackage{amsmath,amsfonts,bm}

\def\eqref#1{equation~\ref{#1}}

\def\1{\bm{1}}

\DeclareMathAlphabet{\mathsfit}{\encodingdefault}{\sfdefault}{m}{sl}
\SetMathAlphabet{\mathsfit}{bold}{\encodingdefault}{\sfdefault}{bx}{n}

\newcommand{\E}{\mathbb{E}}

\newcommand{\R}{\mathbb{R}}

\newcommand{\KL}{D_{\mathrm{KL}}}

\DeclareMathOperator*{\argmin}{arg\,min}

\DeclareMathOperator{\Tr}{Tr}

\usepackage{hyperref}
\usepackage{url}
\usepackage{amsthm}
\usepackage{mathtools}
\usepackage{booktabs}
\usepackage{multirow}
\usepackage{graphicx}
\usepackage{thm-restate}
\usepackage{subcaption} %
\usepackage{makecell}
\usepackage{wrapfig}
\usepackage{float}

\usepackage{tikz}
\usetikzlibrary{arrows.meta,patterns}

\title{Output-aware Residual Stream Pruning for Large Language Models}

\author{Chayne Thrash, Kevin Chen \& Soheil Kolouri \\
Department of Computer Science\\
Vanderbilt University\\
Nashville, TN 37235, USA \\
\texttt{\{chayne.thrash,kevin.l.chen,soheil.kolouri\}@vanderbilt.edu} 
}

\begin{document}

\maketitle

\begin{abstract}
Residual stream pruning methods reduce inference cost by shrinking the model's hidden dimension, but existing approaches typically choose these dimensions by minimizing activation reconstruction error. This criterion implicitly treats all perturbation directions as equally important, ignoring the sensitivity of downstream layers. We introduce a sensitivity-aware approach to residual-stream pruning that directly accounts for this direction-dependent sensitivity. Using a second-order approximation to the output KL divergence, we characterize the effect of a residual-stream perturbation through both its activation covariance and the local sensitivity of the model output. The resulting subspace selection objective couples these two quantities, but is difficult to optimize directly. We derive a tractable spectral upper bound that reduces subspace selection to an eigendecomposition of a sensitivity-weighted covariance matrix, retaining the efficiency and structural simplicity of rotation-based pruning methods. Across several instruction-tuned language model families, our method consistently reduces calibration KL divergence relative to activation-only pruning and improves perplexity and downstream task performance over a range of compression levels. Our results show that preserving activation energy alone is insufficient for residual-stream pruning, and that explicitly accounting for how perturbations propagate to the model output provides a more effective criterion for selecting dimensions to remove.

\end{abstract}

\section{Introduction}

The recent advancements in large language models (LLMs) \citep{gpt, llama_3, mistral, phi3} have led to widespread adoption across a growing range of applications. At the same time, the increasing size of these models has made deployment increasingly expensive, motivating the development of smaller and more efficient models that retain strong performance on complex tasks.

A variety of techniques have been proposed to reduce the computational and memory requirements of LLMs. Quantization \citep{optq, awq} decreases memory requirements by reducing the number of bits used to represent model weights and activations. Low-rank approximation methods \citep{svd_llm, asvd} compress weight matrices using low-rank factorizations, while knowledge distillation \citep{minillm} trains a smaller student model to mimic the behavior of a larger teacher. Structured pruning has also received considerable attention because it can directly reduce both model size and inference cost without requiring customized kernels or costly retraining.

Most structured pruning methods remove neurons, attention heads, or other structures within the self-attention and MLP blocks while leaving the residual-stream dimension unchanged. Residual-stream pruning \citep{slicegpt}, in contrast, directly reduces the input and output dimensions of these blocks. By exploiting the invariance of transformer layers to orthogonal changes of basis in the residual stream, such methods can identify and remove entire residual-stream directions while absorbing the corresponding rotations into adjacent linear layers. \citet{slicegpt} selects these directions to minimize reconstruction error in the residual-stream activations. However, minimizing activation error alone does not account for how strongly downstream computations depend on the removed directions: two perturbations of similar magnitude can have substantially different effects on the model's output distribution.

In this work, we introduce a sensitivity-aware approach to residual-stream pruning that balances the magnitude of the induced activation error with the sensitivity of the model's output to that error. We derive our objective from a second-order approximation to the KL divergence between the output distributions of the original and pruned models, yielding a curvature-weighted measure of residual-stream importance. Although directly optimizing this objective is difficult, we derive a tractable upper bound whose solution requires only a small number of eigendecompositions and can therefore be incorporated into existing residual-stream pruning pipelines. Across multiple model families and compression rates, our approach more effectively preserves the output distribution of the original model and consistently improves downstream performance over activation-only residual-stream pruning.

\section{Related Work}

\paragraph{Pruning Large Language Models.}
The growing computational and memory requirements of large language models (LLMs) have motivated extensive work on post-training pruning. Classical approaches commonly rely on iterative pruning and retraining \citep{mag_pruning,lottery_ticket,movement_prune}, which becomes increasingly costly at LLM scale. Recent methods therefore emphasize one-shot pruning. SparseGPT \citep{sparsegpt} scales second-order techniques such as Optimal Brain Surgeon \citep{optimal_brain} to billion-parameter models, while Wanda \citep{wanda} combines weight and activation magnitudes. However, the resulting unstructured sparsity generally requires specialized kernels or hardware to realize inference speedups.

Structured pruning instead removes entire network components, directly reducing model size and computation. LLM-Pruner \citep{llm_pruner} uses first-order information to estimate the importance of coupled structures, while LLM-Surgeon \citep{llm_surgeon} applies a Kronecker-factored empirical Fisher to prune rows and columns of weight matrices. Other methods use activation statistics or layerwise reconstruction objectives to remove channels, attention heads, or other internal structures \citep{flap,ziplm,osscar}. In contrast, SliceGPT \citep{slicegpt} reduces the dimensionality of the residual stream itself. Exploiting the rotational invariance of Transformer representations, it uses PCA to identify low-variance directions that can be removed, producing smaller dense weight matrices compatible with standard hardware. Our work builds on this framework by additionally accounting for the sensitivity of the model output when selecting residual-stream directions to remove.

\paragraph{Ouput-Aware Model Compression.}
A complementary line of work considers how compression-induced perturbations affect model behavior. Many post-training methods optimize local quantities such as weight magnitude or activation reconstruction error, whereas sensitivity-aware approaches incorporate gradient or curvature information. LLM-Pruner \citep{llm_pruner} uses first-order Taylor approximations for structured importance estimation, while LLM-Surgeon \citep{llm_surgeon} and GFWSVD \citep{gfwsvd} use Kronecker-factored empirical Fisher approximations for structured pruning and low-rank compression, respectively.

More recent methods optimize objectives tied directly to changes in the model's predictive distribution. YAQA \citep{yaqa} derives a post-training quantization objective from a second-order approximation to the full-model KL divergence. EvoPress \citep{evopress} instead uses evolutionary search to optimize pruning and quantization configurations according to output-distribution KL, with T\'yr-the-pruner \citep{tyr} specializing this approach to structured pruning. Our approach is most closely related to YAQA: we use second-order sensitivity information to approximate changes in the model's output distribution and combine it with activation statistics to identify residual-stream directions that can be removed with minimal effect on model behavior.

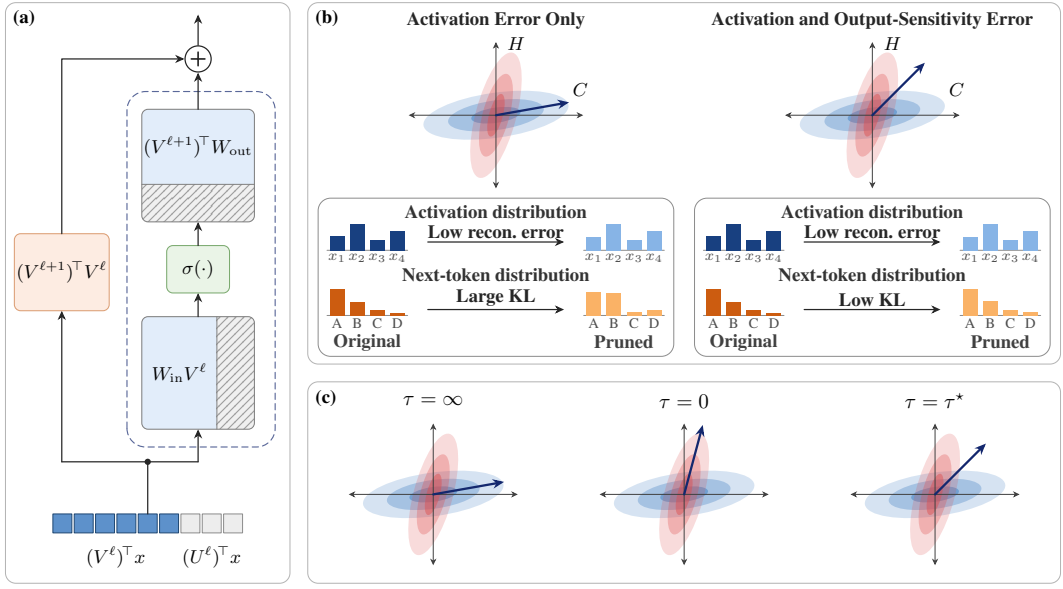
\begin{figure*}[!t]
    \centering
    \resizebox{\textwidth}{!}{\ifcsname tikz@library@arrows.meta@loaded\endcsname\else%
  \errmessage{main_figure.tex: add \string\usetikzlibrary{arrows.meta,patterns} to your preamble}%
\fi%
\ifcsname tikz@library@patterns@loaded\endcsname\else%
  \errmessage{main_figure.tex: add \string\usetikzlibrary{arrows.meta,patterns} to your preamble}%
\fi%
\begingroup%
\definecolor{rspBlue}{RGB}{52,118,190}%
\definecolor{rspRed}{RGB}{222,76,76}%
\definecolor{rspNavy}{RGB}{20,40,110}%
\definecolor{rspBoxBlue}{RGB}{226,237,250}%
\definecolor{rspBoxBlueLine}{RGB}{88,138,205}%
\definecolor{rspBoxPeach}{RGB}{253,236,226}%
\definecolor{rspBoxPeachLine}{RGB}{220,158,116}%
\definecolor{rspBoxGreen}{RGB}{234,244,231}%
\definecolor{rspBoxGreenLine}{RGB}{124,170,112}%
\definecolor{rspBlueDark}{RGB}{24,64,128}%
\definecolor{rspBlueLight}{RGB}{135,180,230}%
\definecolor{rspOrangeDark}{RGB}{204,92,16}%
\definecolor{rspOrangeLight}{RGB}{250,178,102}%
\tikzset{
  flow/.style     ={-{Stealth[length=4.5pt,width=4pt]}, line width=0.6pt, black!85},
  wire/.style     ={line width=0.6pt, black!85},
  axis/.style     ={{Stealth[length=3.2pt,width=3pt]}-{Stealth[length=3.2pt,width=3pt]},
                    line width=0.45pt, black!75},
  dir/.style      ={-{Stealth[length=5.5pt,width=4.5pt]}, line width=1.1pt, rspNavy},
  wbox/.style     ={rounded corners=3pt, line width=0.5pt},
  ghostfill/.style={fill=black!6, rounded corners=3pt},
  ghosthatch/.style={pattern=north east lines, pattern color=black!38, rounded corners=3pt},
  ghostline/.style={draw=black!35, line width=0.4pt, rounded corners=3pt},
  panel/.style    ={rounded corners=5pt, line width=0.45pt, draw=black!30},
  plabel/.style   ={font=\small\bfseries, anchor=north west, inner sep=3pt},
  ptitle/.style   ={font=\footnotesize\bfseries, black!85},
  rowtitle/.style ={font=\footnotesize\bfseries, black!80},
  arrlabel/.style ={font=\footnotesize\bfseries, black!85, inner sep=1.5pt, above},
  mlabel/.style   ={font=\small},
}%
\newcommand{\gaussplot}[4]{%
\begin{scope}[shift={(#1)}, scale=#3]
  \foreach \lv/\al in {1/0.20, 0.66/0.30, 0.34/0.45}{%
    \fill[rspBlue, fill opacity=\al, rotate=8]   (0,0) ellipse ({1.30*\lv} and {0.40*\lv});}
  \foreach \lv/\al in {1/0.20, 0.66/0.30, 0.34/0.45}{%
    \fill[rspRed, fill opacity=\al, rotate=-15] (0,0) ellipse ({0.36*\lv} and {1.15*\lv});}
  \draw[axis] (-1.55,0) -- (1.55,0);
  \draw[axis] (0,-1.30) -- (0,1.30);
  \draw[dir]  (0,0) -- (#2:1.32);
  \ifnum#4=1
    \node[mlabel] at (0.36,1.28) {$H$};
    \node[mlabel] at (1.50,0.46) {$C$};
  \fi
\end{scope}}%
\newcommand{\barchart}[4]{%
\begin{scope}[shift={(#1)}]
  \draw[black!45, line width=0.3pt] (-0.05,0) -- (1.31,0);
  \foreach \h [count=\i from 0] in {#3}{%
    \fill[#2] ({0.34*\i},0) rectangle ({0.34*\i+0.24},{0.44*\h});}
  \foreach \lab [count=\i from 0] in {#4}{%
    \node[font=\scriptsize, black!70, inner sep=0pt] at ({0.34*\i+0.12},-0.13) {\lab};}
\end{scope}}%
\newcommand{\distbox}[3]{%
\begin{scope}[shift={(#1)}]
  \draw[black!45, rounded corners=6pt, line width=0.45pt, fill=white] (-3.0,-1.37) rectangle (3.0,1.3);
  \node[rowtitle] at (0,1.08) {Activation distribution};
  \barchart{-2.80,0.42}{rspBlueDark}{0.55,1,0.40,0.73}{$x_1$,$x_2$,$x_3$,$x_4$}
  \barchart{1.53,0.42}{rspBlueLight}{0.50,1,0.38,0.75}{$x_1$,$x_2$,$x_3$,$x_4$}
  \draw[flow] (-1.15,0.56) -- node[arrlabel] {Low recon.\ error} (1.15,0.56);
  \node[rowtitle] at (0,0.0) {Next-token distribution};
  \barchart{-2.80,-0.68}{rspOrangeDark}{1,0.50,0.20,0.10}{A,B,C,D}
  \barchart{1.53,-0.68}{rspOrangeLight}{#3}{A,B,C,D}
  \draw[flow] (-1.15,-0.58) -- node[arrlabel] {#2} (1.15,-0.58);
  \node[font=\footnotesize\bfseries, black!75] at (-2.17,-1.13) {Original};
  \node[font=\footnotesize\bfseries, black!75] at ( 2.16,-1.13) {Pruned};
\end{scope}}%
\begin{tikzpicture}[line cap=round, line join=round, font=\footnotesize]
\draw[panel] (0,0) rectangle (4.8,9.8);
\node[plabel] at (0.02,9.78) {(a)};
\foreach \i in {0,...,5}{%
  \draw[rspBlue!60!black, line width=0.3pt, fill=rspBlue!80] ({0.8+0.36*\i},0.85) rectangle ++(0.32,0.32);}
\foreach \i in {6,7,8}{%
  \draw[black!45, line width=0.3pt, fill=black!7] ({0.8+0.36*\i},0.85) rectangle ++(0.32,0.32);}
\node[mlabel] at (1.86,0.44) {$(V^{\ell})^{\!\top}x$};
\node[mlabel] at (3.48,0.44) {$(U^{\ell})^{\!\top}x$};
\draw[wire] (2.4,1.17) -- (2.4,2.05);
\draw[wire] (0.95,2.05) -- (3.25,2.05);
\fill[black!85] (2.4,2.05) circle (1.5pt);
\draw[flow] (3.25,2.05) -- (3.25,2.6);
\draw[flow] (0.95,2.05) -- (0.95,4.6);
\draw[densely dashed, rounded corners=9pt, line width=0.55pt, rspNavy!70] (2.05,2.3) rectangle (4.55,8.3);
\begin{scope}
  \clip[rounded corners=3pt] (2.3,2.6) rectangle (4.2,4.5);
  \fill[rspBoxBlue]  (2.3,2.6)   rectangle (3.567,4.5);
  \fill[black!6]  (3.567,2.6) rectangle (4.2,4.5);
  \pattern[pattern=north east lines, pattern color=black!38] (3.567,2.6) rectangle (4.2,4.5);
\end{scope}
\draw[black!50, line width=0.45pt] (3.567,2.6) -- (3.567,4.5);
\draw[black!50, line width=0.5pt, rounded corners=3pt] (2.3,2.6) rectangle (4.2,4.5);
\node[mlabel] at (2.933,3.55) {$W_{\mathrm{in}}V^{\ell}$};
\draw[flow] (3.25,4.5) -- (3.25,4.9);
\draw[rspBoxGreenLine, fill=rspBoxGreen, wbox] (2.72,4.9) rectangle (3.78,5.7);
\node[mlabel] at (3.25,5.3) {$\sigma(\cdot)$};
\draw[flow] (3.25,5.7) -- (3.25,6.1);
\begin{scope}
  \clip[rounded corners=3pt] (2.3,6.1) rectangle (4.2,8.0);
  \fill[rspBoxBlue]  (2.3,6.733) rectangle (4.2,8.0);
  \fill[black!6]  (2.3,6.1)   rectangle (4.2,6.733);
  \pattern[pattern=north east lines, pattern color=black!38] (2.3,6.1) rectangle (4.2,6.733);
\end{scope}
\draw[black!50, line width=0.45pt] (2.3,6.733) -- (4.2,6.733);
\draw[black!50, line width=0.5pt, rounded corners=3pt] (2.3,6.1) rectangle (4.2,8.0);
\node[mlabel] at (3.25,7.367) {$(V^{\ell+1})^{\!\top}W_{\mathrm{out}}$};
\draw[rspBoxPeachLine, fill=rspBoxPeach, wbox] (0.15,4.6) rectangle (1.75,5.9);
\node[mlabel] at (0.95,5.25) {$(V^{\ell+1})^{\!\top}V^{\ell}$};
\node[circle, draw=black!85, line width=0.6pt, minimum size=0.44cm, inner sep=0pt] (plus) at (3.25,8.85) {};
\draw[black!85, line width=0.8pt] (plus.center) ++(-0.115,0) -- ++(0.23,0)
                                   (plus.center) ++(0,-0.115) -- ++(0,0.23);
\draw[flow] (3.25,8.0) -- (plus.south);
\draw[flow] (0.95,5.9) -- (0.95,8.85) -- (plus.west);
\draw[flow] (plus.north) -- (3.25,9.6);
\draw[panel] (5.1,3.7) rectangle (17.8,9.8);
\node[plabel] at (5.12,9.78) {(b)};
\node[ptitle] at (8.275,9.5)  {Activation Error Only};
\node[ptitle] at (14.625,9.5) {Activation and Output-Sensitivity Error};
\gaussplot{8.275,7.9}{10}{0.95}{1}
\gaussplot{14.625,7.9}{45}{0.95}{1}
\distbox{8.275,5.2}{Large KL}{0.90,0.85,0.15,0.20}
\distbox{14.625,5.2}{Low KL}{1,0.55,0.20,0.12}
\draw[panel] (5.1,0) rectangle (17.8,3.4);
\node[plabel] at (5.12,3.38) {(c)};
\foreach \x/\ang/\lab in {7.217/10/{\tau=\infty}, 11.45/75/{\tau=0}, 15.683/45/{\tau=\tau^{\star}}}{%
  \node[font=\normalsize] at (\x,3.08) {$\lab$};
  \gaussplot{\x,1.5}{\ang}{0.92}{0}
}
\end{tikzpicture}%
\endgroup%
}
\caption{Overview of our structured residual-stream pruning method. (a) We remove residual-stream dimensions by rotating unimportant directions into removable coordinates, allowing corresponding columns and rows of the attention and MLP weight matrices to be deleted. (b) Directions are selected using both activation variance and output sensitivity. (c) Our objective balances activation reconstruction and output sensitivity through \(\tau\), interpolating between the two extremes.}
\vspace{-1\baselineskip}
    \label{fig:main}
\end{figure*}

\section{Method}

In this section, we present an output-sensitive approach to residual-stream pruning for large language models. An overview of both residual stream pruning and our proposed method can be seen in Figure~\ref{fig:main}. We first describe residual stream pruning and review the objective used by SliceGPT. We then derive an objective that incorporates both activation reconstruction error and the sensitivity of the model output to removed directions. Finally, we introduce an efficient approximation that generates candidate pruning bases using only a small number of eigendecompositions.

\subsection{Residual-stream pruning}

Orthogonal transformations may be merged into the linear layers along the residual stream allowing for arbitrary rotations to be applied to the activations. We review this reparameterization in Appendix~\ref{sec:orthog_inv}. This invariance can be used to identify bases in which selected directions of the residual stream can
be removed with minimal impact on model performance. Formally, let
\[
    Q
    =
    \begin{bmatrix} V & U \end{bmatrix}
    \in\mathbb{R}^{d\times d},
    \qquad
    V\in\mathbb{R}^{d\times(d-k)},
    \quad
    U\in\mathbb{R}^{d\times k},
\]
where \(V\) spans the directions retained after pruning and \(U\) spans
those to be removed. 

Removing the \(U\)-coordinates leaves the reduced representation
\(V^{\top}x\in\mathbb{R}^{d-k}\). Because the associated projections can
be absorbed into adjacent linear layers, the truncation reduces the
physical width of the model, yielding smaller dense weight matrices.

Due to the residual connection, this construction requires a common orthogonal basis throughout the network. To allow for layer-dependent orthogonal bases, a transition matrix
\[
\left(V^{\ell+1}\right)^{\top}V^{\ell}
    \in\mathbb{R}^{(d-k)\times(d-k)}.
\]
is applied along the residual connection. These transition matrices introduce additional storage and computation along the residual path. Their cost, however, scales as \(\mathcal{O}\bigl((d-k)^2\bigr)\) and decreases quadratically with the retained width.

\section{Direction selection in SliceGPT}

The remaining question is how to determine which directions should be
removed. SliceGPT \citep{slicegpt} selects directions whose removal minimizes the activation reconstruction error on a calibration set. Let
\(s\sim\mathcal{D}_{\mathrm{cal}}\) denote a calibration sequence and let
\(x_i\in\mathbb{R}^{d}\) be the activation associated with its \(i\)-th
token at a fixed layer. Define the activation second-moment matrix
\[
    C
    =
    \mathbb{E}_{s\sim\mathcal{D}_{\mathrm{cal}}}
    \left[
        \frac{1}{|s|}
        \sum_{i=1}^{|s|}x_i x_i^{\top}
    \right].
\]
As shown in Appendix~\ref{sec:sgpt_derivation}, this selection problem
corresponds to uncentered PCA of the residual-stream activations:
\(V\) spans the \(d-k\) leading principal directions retained by the
model, while \(U\) spans the \(k\) trailing directions to be removed.
Equivalently, \(U\) solves
\begin{equation}
\label{eq:slice_gpt_obj}
    \min_{\substack{U\in\mathbb{R}^{d\times k}\\U^{\top}U=I_k}}
    \mathcal{L}_{\mathrm{SG}}(U)
    =
    \Tr\!\left(U^{\top}CU\right).
\end{equation}
This objective measures only the activation energy discarded
at the current layer, without accounting for how perturbations along
different directions are amplified or attenuated downstream. We address
this limitation next by deriving an efficient layer-wise criterion that
incorporates downstream sensitivity into the direction-selection
objective.

\begin{figure*}[t]
    \centering

    \begin{subfigure}[t]{0.485\textwidth}
        \centering\includegraphics[width=\linewidth]{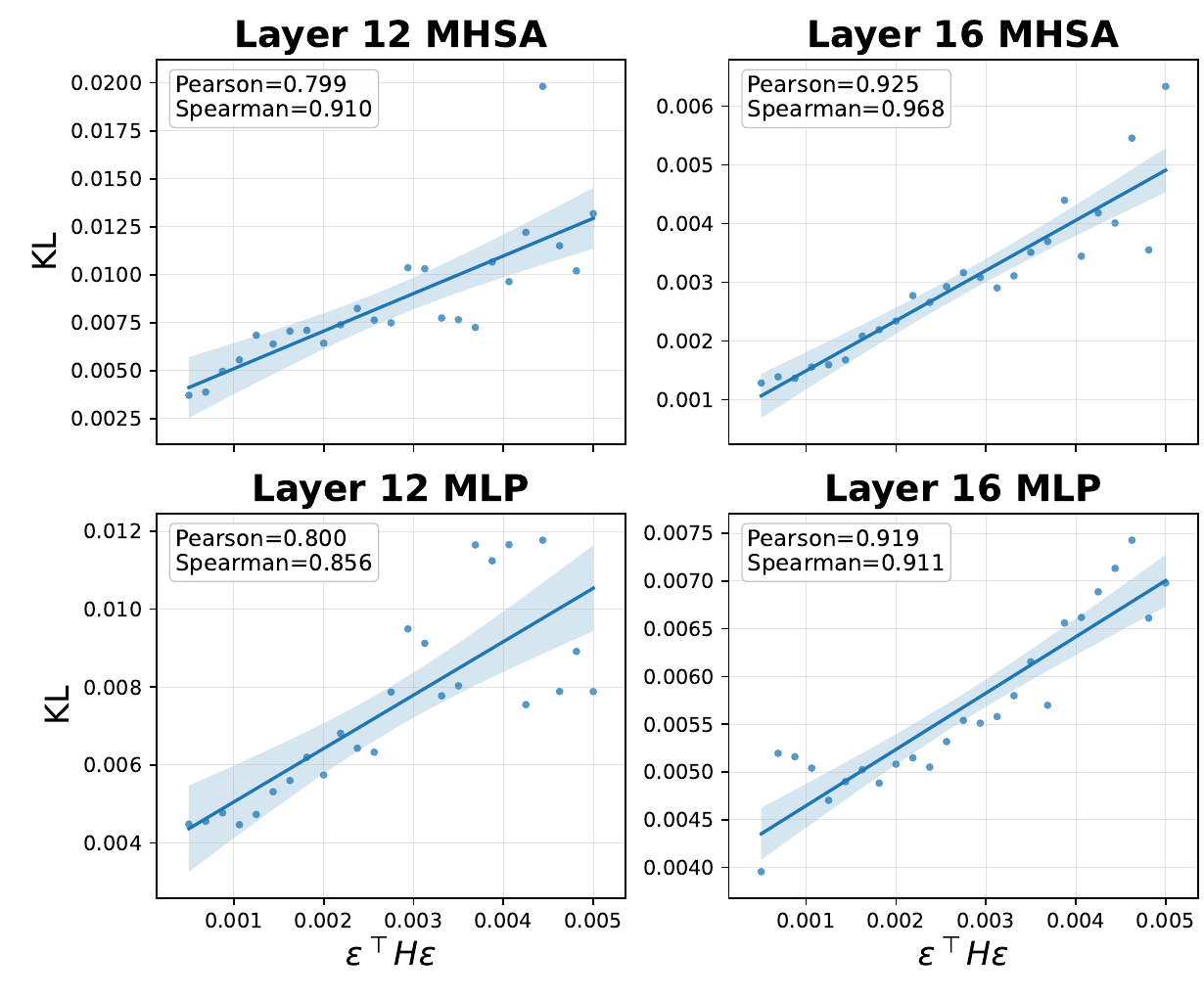}
        \vspace{-2\baselineskip}
        \caption{}
        \label{fig:surrogate_kl}
    \end{subfigure}
    \begin{subfigure}[t]{0.485\textwidth}
        \centering
        \includegraphics[width=\linewidth]{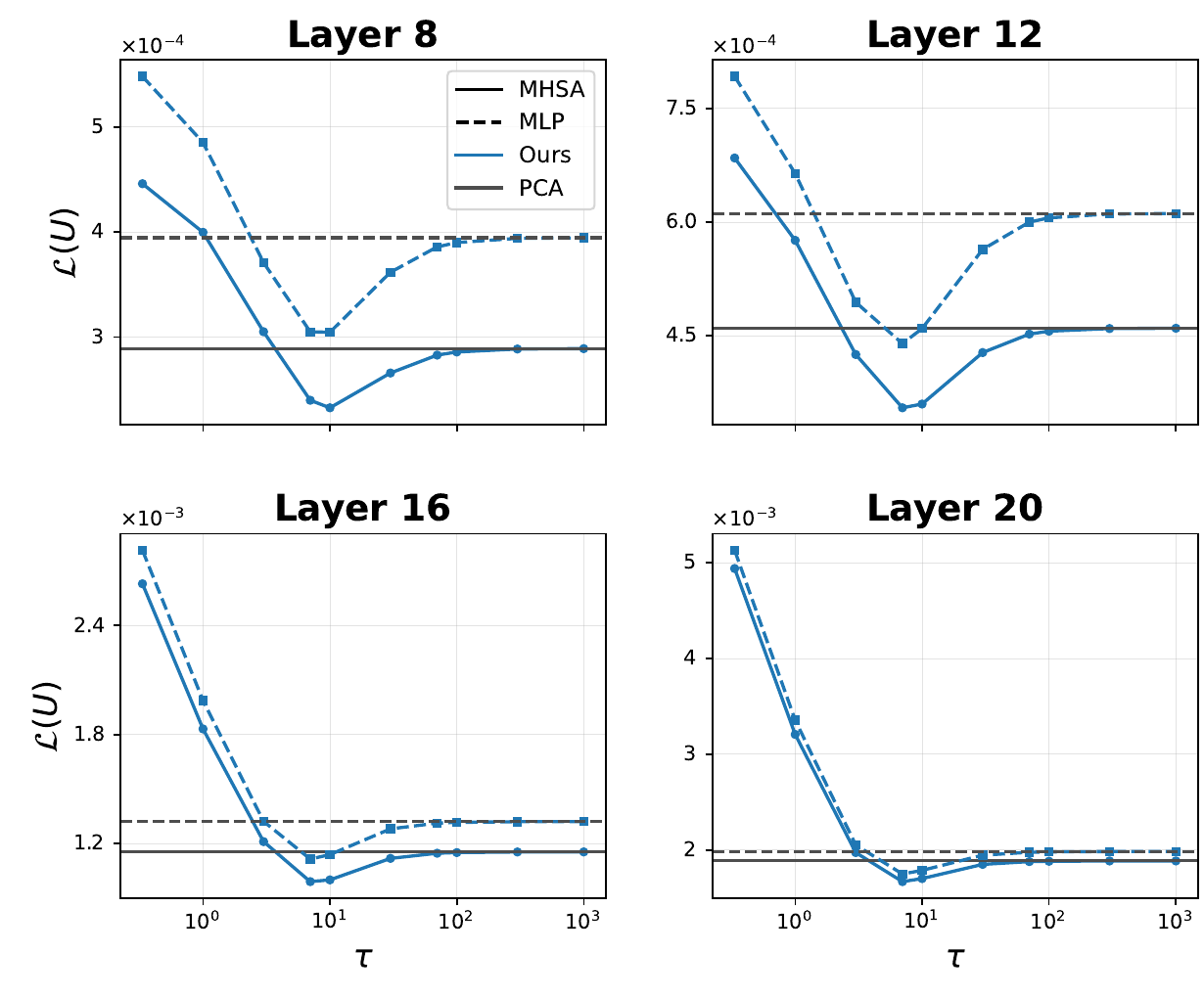}
        \vspace{-2\baselineskip}
        \caption{}
        \label{fig:bound_tau}

    \end{subfigure}
    \caption{
        Validation of the proposed approach on Llama 3.2 3B. (a) measures the correlation between our sensitivity approximation with KL. We fix  activation reconstruction error to the error induced at \(25\%\) pruning rate and vary perturbation direction. (b) shows the value of Equation \ref{eq:objective} attained for bases which minimize Equation \ref{eq:upper_bound} for various values of \(\tau\). Results are shown across different layers at a pruning rate of \(25\%\). We compare with the loss achieved by the SliceGPT solution. 
    }
    \label{fig:surrogate_validation}
\end{figure*}
 
\subsection{Sensitivity-aware direction selection}

SliceGPT selects the removed subspace by minimizing the activation energy
discarded at each layer. We instead select this subspace according to the
effect of the resulting perturbation on the model output. Let
\(\mathcal{U}=\{U^\ell\}_{\ell=1}^{L}\) denote the collection of removed
subspaces across the \(L\) pruned blocks. If
\(\hat{\theta}(\mathcal{U})\) denotes the parameters of the resulting
pruned model, the ideal objective is
\[
    \min_{\mathcal{U}}
    \mathbb{E}_{s\sim\mathcal{D}_{\mathrm{cal}}}
    \left[
        \KL\!\left(
            p_\theta(\cdot\mid s)
            \,\Vert\,
            p_{\hat{\theta}(\mathcal{U})}(\cdot\mid s)
        \right)
    \right],
\]
where \(p_\theta(\cdot\mid s)\) denotes the predictive distribution of
the original model. Jointly optimizing this objective over all pruning
bases requires repeated end-to-end evaluations of the pruned model. We
therefore construct a layerwise surrogate from the local output
sensitivity of the unpruned model.

Fix a layer and suppress its layer index for clarity. Let
\(x_i\in\mathbb{R}^d\) be the residual-stream activation associated with
token \(i\) in a calibration sequence \(s\). For a perturbed activation
\(\hat{x}_i\in\mathbb{R}^d\), let
\(p_\theta(\cdot\mid s;\hat{x}_i)\) denote the predictive distribution
obtained by replacing \(x_i\) with \(\hat{x}_i\) while leaving all other
activations and model parameters unchanged. Define
\[
    D_i(\hat{x}_i;s)
    =
    \KL\!\left(
        p_\theta(\cdot\mid s)
        \,\Vert\,
        p_\theta(\cdot\mid s;\hat{x}_i)
    \right).
\]
Assuming \(D_i\) is twice differentiable at \(\hat{x}_i=x_i\), its local
expansion is
\[
    D_i(\hat{x}_i;s)
    =
    \frac{1}{2}
    (\hat{x}_i-x_i)^\top
    H_i(s)
    (\hat{x}_i-x_i)
    +
    o\!\left(\lVert\hat{x}_i-x_i\rVert_2^2\right),
\]
where \(H_i(s)
    =
    \left.
    \nabla_{\hat{x}_i}^2 D_i(\hat{x}_i;s)
    \right|_{\hat{x}_i=x_i}
    \in\mathbb{R}^{d\times d}\).
The constant and first-order terms vanish because
\(D_i(x_i;s)=0\) is a minimum of the KL divergence. Under the usual regularity conditions, \(H_i(s)\) is the Fisher
information matrix with respect to the intervened activation.

Perturbing all token activations jointly yields a Hessian over the concatenated sequence representation requiring storage of \(|s|^{2}\) \(d \times d\) matrices per sequence which is impractical. We therefore discard the cross-token blocks. Together with our layerwise treatment, this amounts to neglecting both cross-token and cross-layer curvature interactions.

Let \(U\in\mathbb{R}^{d\times k}\), with \(U^\top U=I_k\), span the
subspace removed at the current layer. Orthogonal projection onto the
retained subspace gives
\[
    \hat{x}_i
    =
    \left(I_d-UU^\top\right)x_i,
    \qquad
    x_i-\hat{x}_i
    =
    UU^\top x_i.
\]
Substituting this perturbation into the token-separable quadratic
approximation and omitting the constant factor \(1/2\) gives
\begin{align}
    \mathcal{L}_{\mathrm{local}}(U)
    &=
    \mathbb{E}_{s\sim\mathcal{D}_{\mathrm{cal}}}
    \left[
        \frac{1}{|s|}
        \sum_{i=1}^{|s|}
        (UU^\top x_i)^\top
        H_i(s)
        (UU^\top x_i)
    \right]
    \nonumber\\
    &=
    \mathbb{E}_{s\sim\mathcal{D}_{\mathrm{cal}}}
    \left[
        \frac{1}{|s|}
        \sum_{i=1}^{|s|}
        \Tr\!\left(
            U^\top x_i x_i^\top U\,
            U^\top H_i(s)U
        \right)
    \right].
    \label{eq:tokenwise_curvature_objective}
\end{align}

Using token-specific curvature in a deterministic optimization procedure
would require storing or repeatedly recomputing a \(d\times d\) matrix
for every calibration token. We therefore replace \(H_i(s)\) by the
shared layerwise curvature
\[
    H
    =
    \mathbb{E}_{s\sim\mathcal{D}_{\mathrm{cal}}}
    \left[
        \frac{1}{|s|}
        \sum_{i=1}^{|s|}H_i(s)
    \right].
\]
This approximation removes the token dependence of the curvature and
therefore discards correlations between the activation outer products
\(x_i x_i^\top\) and their corresponding curvature matrices \(H_i(s)\). Substituting the shared
curvature \(H\) into \eqref{eq:tokenwise_curvature_objective} yields our
final direction-selection objective:
\begin{equation}
    \label{eq:objective}
    \min_{\substack{U\in\mathbb{R}^{d\times k}\\U^\top U=I_k}}
    \mathcal{L}(U)
    =
    \Tr\!\left(
        U^\top C U\,
        U^\top H U
    \right),
\end{equation}
The objective requires storing only the two \(d\times d\) layerwise
statistics \(C\) and \(H\), with \(H\) being the only additional matrix
relative to SliceGPT. Additionally, when \(H=I_d\), the orthonormality of \(U\) exactly recovers equation~\ref{eq:slice_gpt_obj} which is the SliceGPT objective. We additionally examine whether \(H\) captures directions to which the model output is particularly sensitive in Figure~\ref{fig:surrogate_kl}. Holding activation reconstruction error fixed and varying only the perturbation direction, we observe that directions assigned larger error by the curvature estimate generally induce larger KL divergence. This suggests that \(H\) captures output sensitivity that cannot be distinguished using reconstruction error alone. 

Optimizing this objective for a general positive
semidefinite \(H\), however, \eqref{eq:objective} does not reduce directly
to a standard eigenspace problem. Although it can be addressed through
iterative optimization on the Stiefel manifold, doing so introduces
substantial computational overhead. In the following section, we derive
a spectral upper bound that yields a tractable approximation using only
a small number of eigendecompositions.

\subsection{Efficient approximation of the objective and its theoretical analysis}

Directly minimizing \eqref{eq:objective} over a multidimensional subspace is nontrivial. We first consider removing a single unit direction $u$, for which the loss becomes
$\mathcal{L}(u)=(u^\top C u)(u^\top H u)$. We show that this problem admits an equivalent spectral formulation, reducing its solution to a one-dimensional search over the smallest eigenvalue of a weighted sum of $C$ and $H$. This exact rank-one characterization motivates our extension to higher-dimensional pruning subspaces. We defer all proofs to Appendix~\ref{sec:derivations}

\begin{restatable}{proposition}{propone}
\label{prop_one}
Let \(C,H\in\R^{d\times d}\) be symmetric positive semidefinite
matrices. Then
\begin{equation}
\label{eq:rankone_equivalence}
    \min_{\substack{u\in\R^d\\ \|u\|_2=1}}
    (u^\top C u)(u^\top H u)
    =
    \frac14
    \inf_{\tau>0}
    \left[
        \lambda_{\min}\!\left(
            \tau C+\tau^{-1}H
        \right)
    \right]^2.
\end{equation}
If \(C,H\succ0\), the infimum is attained at some
\(\tau_\star>0\), and any unit eigenvector associated with the
smallest eigenvalue of \(\tau_\star C+\tau_\star^{-1}H\)
globally minimizes the rank-one pruning objective.
\end{restatable}

Motivated by this exact rank-one characterization, we extend the
spectral construction to removing a $k$-dimensional subspace. For each
$\tau>0$, we form $U$ from the eigenvectors associated with the $k$
smallest eigenvalues of $M_\tau=\tau C+\tau^{-1}H$. We show that this
basis minimizes an upper bound on the loss in \eqref{eq:objective}.
Equivalently, it minimizes a regularized version of $\mathcal{L}(U)$,
with a nonnegative regularization term equal to the gap between the
upper bound and the original loss.

\begin{restatable}{proposition}{proptwo}
\label{prop_two}
Let \(C,H\in\R^{d\times d}\) be symmetric positive semidefinite
matrices, and let \(1\leq k\leq d\). For every \(\tau>0\), define
\(M_\tau=\tau C+\tau^{-1}H\). Then, for any
\(U\in\R^{d\times k}\) satisfying \(U^\top U=I_k\),
\begin{equation}
\label{eq:upper_bound}
    \mathcal{L}(U)
    =
    \Tr\!\left(U^\top CU\,U^\top HU\right)
    \leq
    \underbrace{
        \frac14\Tr\!\left(U^\top M_\tau^2 U\right)
    }_{\mathcal{B}_\tau(U)}.
\end{equation}
For fixed \(\tau\), this upper bound is globally minimized by
choosing the columns of \(U\) as eigenvectors associated with
the \(k\) smallest eigenvalues of \(M_\tau\), yielding
\begin{equation}
\label{eq:spectral_bound_minimum}
    \min_{U^\top U=I_k}\mathcal{B}_\tau(U)
    =
    \frac14\sum_{j=1}^k\lambda_j(M_\tau)^2,
\end{equation}
where the eigenvalues are ordered increasingly.
\end{restatable}

We next show that jointly optimizing this upper bound over $U$ and
$\tau$ is equivalent to minimizing the original loss $\mathcal{L}(U)$
augmented by the regularization term
$\|U^\top CU\|_F\|U^\top HU\|_F$, which couples activation energy
and output sensitivity within the removed subspace.

\begin{restatable}{proposition}{propthree}
\label{prop_three}
Let \(C,H\in\R^{d\times d}\) be symmetric positive semidefinite
matrices, and define \(M_\tau=\tau C+\tau^{-1}H\).
For \(1\leq k\leq d\),
\begin{equation}
\label{eq:regularized_equivalence}
\begin{aligned}
    \frac14\inf_{\tau>0}
    \sum_{j=1}^k\lambda_j(M_\tau)^2
    =
    \min_{U^\top U=I_k}
    \Bigl[
        \frac12\mathcal{L}(U)
        +\frac12
        \|U^\top CU\|_F\|U^\top HU\|_F
    \Bigr],
\end{aligned}
\end{equation}
where the eigenvalues are ordered increasingly.
If \(C,H\succ0\), the infimum is attained.

Thus, jointly optimizing the spectral upper bound over the
subspace and the scale \(\tau\) yields the equivalent
regularized problem
\begin{equation}
\label{eq:regularized_objective}
    \underset{U^\top U=I_k}{\arg\min}
    \left[
        \mathcal{L}(U)
        +
        \|U^\top CU\|_F\|U^\top HU\|_F
    \right],
\end{equation}
where the common factor \(1/2\) has been omitted.
\end{restatable}

These results suggest a tractable eigendecomposition approximation to the original pruning objective. While the rank-one case is solved exactly, the higher-dimensional construction minimizes an upper bound that can be interpreted as a regularized form of the desired loss. This yields an efficient basis construction approach described next.

\paragraph{Candidate selection.}
The preceding analysis motivates a spectral family of candidate
subspaces, which we evaluate using the original objective.
Given a grid
\(\mathcal{T}=\{\tau_j\}_{j=1}^{n}\subset(0,\infty)\), we compute
\[
    U_j=\operatorname{eig}_{\min,k}(M_{\tau_j}),
    \qquad j=1,\ldots,n,
\]
and select \(U_{j^\star}\), where
\[
    j^\star
    \in
    \argmin_{1\leq j\leq n}
    \mathcal{L}(U_j)
    =
    \argmin_{1\leq j\leq n}
    \Tr\!\left(
        U_j^\top C U_j\,
        U_j^\top H U_j
    \right).
\]
This procedure requires \(n\) independent eigendecompositions,
which can be parallelized to reduce wall-clock time at the cost
of additional memory. In our experiments, a grid of \(n=5\) values of \(\tau\) is sufficient to achieve strong performance.

We validate the candidate construction approach in Figure~\ref{fig:bound_tau}. Candidates obtained by minimizing \(\mathcal{B}_\tau(U)\) across different values of \(\tau\) attain substantially different values of the original objective \(\mathcal{L}(U)\), and an appropriate choice of \(\tau\) consistently yields lower \(\mathcal{L}(U)\) than the SliceGPT solution based only on \(C\), supporting the use of minimizing \(\mathcal{B}_\tau(U)\) as an efficient mechanism for generating candidate subspaces.

\section{Experiments}
\label{sec:experiments}

\begin{table*}[t]
\centering
\scriptsize
\setlength{\tabcolsep}{2.5pt} %
\renewcommand{\arraystretch}{1.15}
\caption{
Comparison across model families and compression rates.
We report KL divergence, perplexity (PPL), and average downstream accuracy (Acc.).
}
\label{tab:main_results}
\begin{tabular}{llccc ccc ccc ccc ccc}
\toprule
& &
\multicolumn{6}{c}{\textbf{Llama-3.x Instruct}} &
\multicolumn{6}{c}{\textbf{Mistral Instruct}} &
\multicolumn{3}{c}{\textbf{Phi-3 Instruct}} \\
\cmidrule(lr){3-8}
\cmidrule(lr){9-14}
\cmidrule(lr){15-17}

\textbf{Sparsity}
& \textbf{Method}
& \multicolumn{3}{c}{\textbf{3.2 3B}}
& \multicolumn{3}{c}{\textbf{3.1 8B}}
& \multicolumn{3}{c}{\textbf{7B v0.3}}
& \multicolumn{3}{c}{\textbf{Nemo}}
& \multicolumn{3}{c}{\textbf{Medium}} \\
\cmidrule(lr){3-5}
\cmidrule(lr){6-8}
\cmidrule(lr){9-11}
\cmidrule(lr){12-14}
\cmidrule(lr){15-17}

& &
\(\mathbf{\KL}\) & \textbf{PPL} & \textbf{Acc.} &
\(\mathbf{\KL}\) & \textbf{PPL} & \textbf{Acc.} &
\(\mathbf{\KL}\) & \textbf{PPL} & \textbf{Acc.} &
\(\mathbf{\KL}\) & \textbf{PPL} & \textbf{Acc.} &
\(\mathbf{\KL}\) & \textbf{PPL} & \textbf{Acc.} \\
\midrule
0\%
& N/A
& 0.0 & 11.76 & 60.55
& 0.0 & 7.21 & 68.53
& 0.0 & 5.49 & 69.72
& 0.0 & 6.09 & 70.05
& 0.0 & 4.30 & 72.95 \\
\midrule

\multirow{2}{*}{10\%}
& SliceGPT
& 0.5059 & 19.90 & 58.11
& 0.7507 & 15.63 & 65.78
& 0.1944 & 6.58 & 67.91
& 0.2791 & 8.11 & 67.06
& 0.5192 & 6.38 & 72.84 \\
& Ours
& 0.2017 & 14.30 & 58.53
& 0.4040 & 11.04 & 66.27
& 0.1737 & 6.47 & 68.02
& 0.2170 & 7.62 & 67.76
& 0.4373 & 6.05 & 72.56 \\
\midrule

\multirow{2}{*}{20\%}
& SliceGPT
& 0.9729 & 31.40 & 53.49
& 1.3165 & 27.87 & 61.06
& 0.4883 & 8.75 & 64.25
& 0.5991 & 11.23 & 62.28
& 1.0421 & 10.68 & 68.63 \\
& Ours
& 0.5139 & 19.67 & 54.92
& 0.8820 & 18.04 & 62.13
& 0.4626 & 8.58 & 65.02
& 0.5254 & 10.44 & 62.92
& 0.8021 & 8.43 & 69.54 \\
\midrule

\multirow{2}{*}{25\%}
& SliceGPT
& 1.2568 & 42.54 & 50.57
& 1.6814 & 40.24 & 57.97
& 0.7065 & 10.81 & 61.23
& 0.8002 & 13.73 & 58.92
& 1.2801 & 13.56 & 65.17 \\
& Ours
& 0.7290 & 24.37 & 51.95
& 1.1569 & 23.79 & 59.02
& 0.6711 & 10.52 & 62.52
& 0.6982 & 12.40 & 60.02
& 1.0074 & 10.27 & 67.58 \\
\midrule

\multirow{2}{*}{30\%}
& SliceGPT
& 1.5830 & 57.17 & 47.40
& 2.0086 & 55.90 & 53.99
& 0.9751 & 14.06 & 57.40
& 1.0526 & 17.66 & 54.03
& 1.4912 & 16.74 & 61.78 \\
& Ours
& 0.9435 & 30.14 & 49.43
& 1.4304 & 31.32 & 55.72
& 0.9140 & 13.34 & 59.30
& 0.8931 & 15.04 & 55.80
& 1.2336 & 12.84 & 64.90 \\

\bottomrule
\end{tabular}
\end{table*}

\begin{table*}[t]
\centering
\footnotesize
\setlength{\tabcolsep}{4pt} %
\renewcommand{\arraystretch}{1}
\caption{
Comparison of SliceGPT and our method on mathematical reasoning and instruction following tasks.
For IFEval, we report strict prompt-level (Prompt) and
instruction-level (Inst.) accuracy. For GSM8K, we report strict extraction accuracy for Llama models and flexible extraction for Mistral.
All results are reported as percentages.
}
\label{tab:generation_results}

\begin{tabular}{llccccccccc}
\toprule
& &
\multicolumn{3}{c}{\textbf{Llama-3.2 3B Instruct}} &
\multicolumn{3}{c}{\textbf{Llama-3.1 8B Instruct}} &
\multicolumn{3}{c}{\textbf{Mistral Nemo Instruct}}\\
\cmidrule(lr){3-5} \cmidrule(lr){6-8} \cmidrule(lr){9-11}

\textbf{Sparsity}
& \textbf{Method}
& \multicolumn{2}{c}{\textbf{IFEval}}
& \textbf{GSM8K}
& \multicolumn{2}{c}{\textbf{IFEval}}
& \textbf{GSM8K}
& \multicolumn{2}{c}{\textbf{IFEval}}
& \textbf{GSM8K}\\

\cmidrule(lr){3-4}
\cmidrule(lr){6-7}
\cmidrule(lr){9-10}

& & \textbf{Prompt} & \textbf{Inst.} & \textbf{Strict}
  & \textbf{Prompt} & \textbf{Inst.} & \textbf{Strict}
  & \textbf{Prompt} & \textbf{Inst.} & \textbf{Flexible}\\

\midrule

0\% & N/A & 71.4 & 79.5 & 76.6 & 74.5 & 81.8 & 85.3 & 56.0 & 67.4 & 80.0\\

\midrule

\multirow{2}{*}{10\%}
& SliceGPT & 58.8 & 69.3 & 66.0 & 63.4 & 72.4 & 80.3 & 52.7 & 63.1 & 77.2 \\
& Ours     & \textbf{63.4} & \textbf{73.9} & \textbf{69.7} & \textbf{65.8} & \textbf{75.7} & \textbf{81.4} & \textbf{54.7} & \textbf{65.4} & \textbf{78.5}\\
\midrule

\multirow{2}{*}{20\%}
& SliceGPT & 50.5 & 62.4 & 61.6 & 54.2 & 66.0 & 74.0 & 47.3 & 58.4 & 67.1 \\
& Ours     & \textbf{56.2} & \textbf{67.5} & \textbf{64.8} & \textbf{56.8} & \textbf{67.5} & \textbf{75.1} & \textbf{47.9} & \textbf{60.4} & \textbf{69.1} \\
\midrule

\multirow{2}{*}{25\%}
& SliceGPT & 49.7 & 56.1 & 54.7 & 47.3 & 60.1 & 63.3 & 38.8 & 51.4 & 55.7 \\
& Ours     & \textbf{51.0} & \textbf{61.5} & \textbf{60.1} & \textbf{51.8} & \textbf{61.8} & \textbf{65.8} & \textbf{45.8} & \textbf{57.0}  & \textbf{61.7} \\
\midrule

\multirow{2}{*}{30\%}
& SliceGPT & 37.5 & 51.0 & 47.1 & 42.7 & 52.3 & 52.8 & 30.9 & 44.4 & 42.2 \\
& Ours     & \textbf{44.9} & \textbf{55.0} & \textbf{51.1} & \textbf{44.2} & \textbf{56.2} & \textbf{55.1} & \textbf{37.5} & \textbf{50.0} & \textbf{47.5} \\

\bottomrule
\end{tabular}
\vspace{-1\baselineskip}
\end{table*}

\begin{figure*}[t]
\centering
\includegraphics[width=\linewidth]{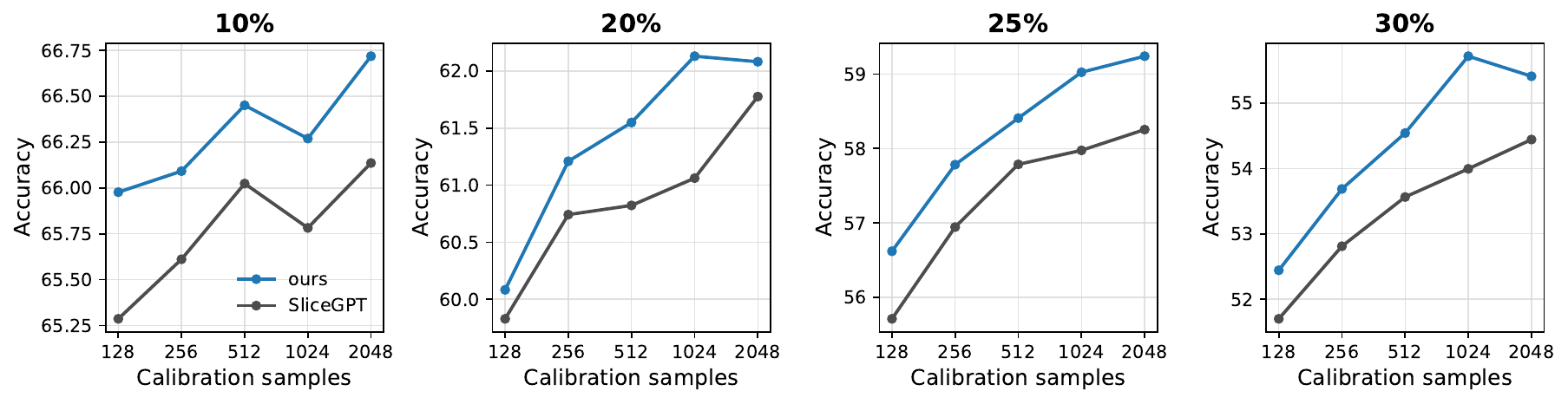}
\vspace{-1\baselineskip}
\caption{
Average common-sense reasoning performance of the proposed method and SliceGPT using varying numbers of calibration samples on Llama-3.1 8B Instruct.
}
\label{fig:calibration_size}
\vspace{-1\baselineskip}
\end{figure*}

We evaluate our method across five instruction-tuned LLMs from three model families. Our primary baseline is SliceGPT~\citep{slicegpt}, the most directly related residual-stream pruning method, with additional comparisons to structured approaches that prune components within the attention and MLP blocks. Implementation details are provided in Appendix~\ref{sec:implementation}.

We evaluate language modeling, common-sense reasoning, and generative performance. For language modeling, we use WikiText-2~\citep{wikitext} and, following \cite{yaqa}, report perplexity (PPL) and KL divergence from the original model. Common-sense reasoning is evaluated on WinoGrande~\citep{winogrande}, MMLU~\citep{mmlu}, ARC-Easy and ARC-Challenge~\citep{arc}, HellaSwag~\citep{hellaswag}, OpenBookQA~\citep{obqa}, and PIQA~\citep{piqa}. Generative evaluation uses GSM8K~\citep{gsm8k} and IFEval~\citep{ifeval}.

For language modeling and common-sense reasoning, we calibrate using 1,024 FineWeb-Edu~\citep{fineweb} sequences of 2,048 tokens. For generative tasks, we instead use 1,024 Tulu3~\citep{tulu3} user--assistant examples with the same maximum sequence length. Because free-form generation may exhibit different output sensitivities than language-modeling data, we compute the sensitivity objective only over assistant-response tokens, so the estimated KL sensitivity reflects perturbations to generated outputs.

\paragraph{Language modeling and common-sense reasoning results} We compare our method with SliceGPT across sparsity ratios of 10\%, 20\%, 25\%, and 30\%, where sparsity denotes the fraction of the residual-stream hidden dimension removed. Table~\ref{tab:main_results} reports language modeling and common-sense reasoning performance across the Llama \citep{llama_3}, Mistral \citep{mistral}, and Phi \citep{phi3} model families. Our method consistently better preserves the behavior of the original model, achieving lower KL divergence and perplexity than SliceGPT across every model and compression rate evaluated. These improvements also generally translate to downstream performance: our method achieves higher average common-sense reasoning accuracy in almost all settings. The benefit of incorporating output sensitivity becomes particularly pronounced as the compression rate increases. For example, at 30\% sparsity, the perplexity of Llama-3.2-3B is reduced from 57.17 with SliceGPT to 30.14 with our method, while average accuracy improves from 47.40 to 49.43. Improvements are also observed across both Mistral models and Phi-3 Medium, demonstrating that the proposed criterion generalizes across different model families.

\paragraph{Mathematical Reasoning and Instruction-Following.}
We additionally compare our method with SliceGPT on mathematical reasoning and instruction-following across multiple sparsity levels. Results for the instruction-tuned Llama-3.2-3B, Llama-3.1-8B, and Mistral-Nemo models are reported in Table~\ref{tab:generation_results}. Across all settings, our method better preserves both mathematical reasoning and instruction-following performance, with the gap generally increasing at higher sparsity. These results provide further evidence that preserving the model's output distribution helps retain capabilities beyond next-token prediction, including mathematical reasoning and instruction following.

\paragraph{Effect of calibration set size} Post-training pruning performance depends on the calibration data used. In our main experiments, we use 1,024 calibration sequences for all methods. To evaluate sensitivity to calibration set size, we vary the number of FineWeb-Edu sequences and measure average common-sense reasoning accuracy across multiple sparsity rates. As shown in Figure~\ref{fig:calibration_size}, our method consistently outperforms SliceGPT across all settings. Both methods generally benefit from additional data, with the largest gains occurring between 128 and 512--1,024 sequences. Beyond 1,024 sequences, improvements are small and occasionally non-monotonic, suggesting that 1,024 samples are sufficient for stable pruning decisions in this setting.

\begin{table}[htbp]
\centering
\centering
\small
\setlength{\tabcolsep}{2.5pt}
\renewcommand{\arraystretch}{1}
\caption{
Effect of additional optimization of \(\mathcal{L}(U)\) on Llama-3.1-8B-Instruct.
\(\KL^{\mathrm{cal}}\) and \(\KL^{\mathrm{wiki}}\) denote KL divergence on the
calibration data and WikiText-2 test set, respectively.
}
\label{tab:opt_gap_results}
\begin{tabular}{cccccc}
\toprule
\textbf{Sparsity}
&
\textbf{Add. Opt.}
&
\(\mathbf{\KL^{\text{cal}}}\)
&
\(\mathbf{\KL^{\text{wiki}}}\)
&
\textbf{PPL}
&
\textbf{Avg. Acc.}

\\
\midrule
\multirow{2}{*}{10\%} & $\times$ & 0.1369 & 0.4040 & 11.04 & 66.27 \\
& $\checkmark$ & 0.1371 & 0.4038 & 11.11 & 66.37 \\
\midrule
\multirow{2}{*}{20\%} & $\times$ & 0.3289 & 0.8820 & 18.04 & 62.13 \\
& $\checkmark$ & 0.3252 & 0.9178 & 18.75 & 62.73 \\
\midrule
\multirow{2}{*}{25\%} & $\times$ & 0.4611 & 1.1569 & 23.79 & 59.02 \\
& $\checkmark$ & 0.4520 & 1.1856 & 24.57 & 59.44 \\
\midrule
\multirow{2}{*}{30\%} & $\times$ & 0.6206 & 1.4304 & 31.32 & 55.72 \\
& $\checkmark$ & 0.6059 & 1.4487 & 31.91 & 55.99 \\
\bottomrule
\end{tabular}
\end{table}

\paragraph{Effect of additional optimization.}
The selected subspace to remove comes from minimizing an upper bound of the desired objective rather than Equation~\ref{eq:objective} itself. We thus investigate whether further optimizing the selected subspace improves downstream performance. Table~\ref{tab:opt_gap_results} shows the results of this experiment using 10,000 optimization steps per pruning site. As can be seen, additional optimization yields only small reductions in calibration KL and modest accuracy improvements. This suggests that the spectral candidate selection already produces a strong solution and that further minimizing the calibration objective provides limited benefit.

\begin{table}[htbp]
\centering
    \centering
    \small
    \setlength{\tabcolsep}{3pt}
    \renewcommand{\arraystretch}{1}

\caption{
    Calibration runtime (HH:MM) comparison on a single NVIDIA H100 GPU.
    Direction selection includes layerwise covariance accumulation
    and pruning-basis construction.
}
    \label{tab:runtime}

    \begin{tabular}{llccc}
        \toprule
        \textbf{Model}
        & \textbf{Method}
        & \makecell{\textbf{Sensitivity}\\\textbf{Estimation}}
        & \makecell{\textbf{Direction}\\\textbf{Selection}}
        & \textbf{Total} \\
        \midrule

        \multirow{2}{*}{\makecell{Llama-3.1\\8B}}
        & SliceGPT
        & --
        & 01:03
        & 01:03 \\
        & Ours
        & 00:09
        & 01:05
        & 01:14 \\
        \midrule

        \multirow{2}{*}{\makecell{Phi-3\\Medium}}
        & SliceGPT
        & --
        & 01:57
        & 01:57 \\
        & Ours
        & 00:18 
        & 02:01
        & 02:19 \\

        \bottomrule
    \end{tabular}
\end{table}

\paragraph{Calibration time comparison.}
Our method adds two sources of overhead relative to SliceGPT: sensitivity estimation and additional eigendecompositions during pruning-direction selection. Table~\ref{tab:runtime} reports calibration runtimes. For both methods, runtime is dominated by direction selection, largely due to double-precision activation covariance accumulation. On Llama-3.1-8B, sensitivity estimation takes 9 minutes, accounting for only \(12\%\) of our total calibration time, while the additional eigendecompositions add comparatively little overhead. Overall, calibration time increases only modestly with the incorporation of output sensitivity.

\begin{table}[htbp]
\centering
    \centering
    \small
    \setlength{\tabcolsep}{3pt}
    \renewcommand{\arraystretch}{1}
    \caption{
        Average common-sense reasoning accuracy of additional structured pruning methods on Llama-3.1-8B-Instruct across pruning rates.
    }
    \label{tab:additional_results}

    \begin{tabular}{lcccc}
        \toprule
        \textbf{Method}
        & \textbf{10\%}
        & \textbf{20\%}
        & \textbf{25\%}
        & \textbf{30\%} \\
        \midrule

        LLM-Pruner       & 62.91 & 53.97 & 45.08 & 42.98 \\
        OSSCAR       & 64.51 & 58.74 & 55.16 & 46.27 \\
        Wanda-sp       & 65.96 & 59.58 & 52.91 & 49.78 \\
        FLAP       & 63.59 & 58.14 & 53.72 & 51.21 \\
        SliceGPT       & 65.78 & 61.06 & 57.97 & 53.99 \\
        Ours  & 66.27 & 62.13 & 59.02 & 55.72 \\

        \bottomrule
    \end{tabular}
\end{table}

\paragraph{Comparison with additional structured pruning methods.}
We additionally compare residual-stream pruning with structured pruning methods that remove parameters within individual transformer blocks. As shown in Table~\ref{tab:additional_results}, residual-stream pruning remains competitive across the evaluated sparsity rates. However, this comparison is not parameter-equivalent, as the reported sparsity rates exclude the additional parameters introduced by the residual-stream rotations. We therefore view these results primarily as context for the relative behavior of the two pruning paradigms rather than a direct comparison. Reducing this overhead through structured rotations or parameter sharing across blocks represents a promising direction for future work.

\section{Conclusion}

We introduced a sensitivity-aware approach to residual-stream pruning that explicitly accounts for the effect of removed activation directions on the model output distribution. Our formulation combines activation covariance with output sensitivity and admits an efficient approximation utilizing only a few eigendcompositions. Across multiple instruction-tuned language models and evaluation settings, the resulting bases consistently preserve model quality more effectively than covariance-only residual-stream pruning, particularly as sparsity increases. By incorporating output sensitivity directly into the basis construction, our method provides a principled way to identify low-dimensional residual subspaces that better preserve model behavior. This suggests several directions for future work, including shared or more parameter-efficient transformations.

\section*{Acknowledgements}
This work was supported in part by the National Center on Generative AI for Uplifting STEM+C Education (GENIUS Center) and by the United States Air Force under Contract No. FA8750-23-C-0518. The authors also gratefully acknowledge the computing resources provided by the NVIDIA Academic Grant Program.

\bibliography{main}
\bibliographystyle{main}

\appendix
\section{Derivations}
\label{sec:derivations}

\subsection{SliceGPT equivalence to equation 1}
\label{sec:sgpt_derivation}

Let \(U \in \mathbb{R}^{d\times k}\) denote an orthonormal basis for the
removed subspace, with \(U^\top U = I_k\). Since the retained and removed
subspaces are orthogonal complements, the reconstruction error induced by
pruning is
\[
x_i - VV^\top x_i = UU^\top x_i.
\]
SliceGPT minimizes the expected squared reconstruction error over calibration
activations:
\begin{align}
\mathcal{L}_{\mathrm{SG}}(U)
&=
\E_{s\sim\mathcal{D}_{\mathrm{cal}}}
\left[
\frac{1}{|s|}
\sum_{i=1}^{|s|}
\left\|UU^\top x_i\right\|_2^2
\right]
\\
&=
\E_{s\sim\mathcal{D}_{\mathrm{cal}}}
\left[
\frac{1}{|s|}
\sum_{i=1}^{|s|}
x_i^\top UU^\top x_i
\right]
\\
&=
\E_{s\sim\mathcal{D}_{\mathrm{cal}}}
\left[
\frac{1}{|s|}
\sum_{i=1}^{|s|}
\Tr\!\left(
U^\top x_i x_i^\top U
\right)
\right].
\end{align}
Defining the uncentered activation covariance
\[
C =
\E_{s\sim\mathcal{D}_{\mathrm{cal}}}
\left[
\frac{1}{|s|}
\sum_{i=1}^{|s|}
x_i x_i^\top
\right],
\]
we obtain
\[
\mathcal{L}_{\mathrm{SG}}(U)
=
\Tr\!\left(U^\top C U\right).
\]

By the Ky Fan minimum principle,
\[
\min_{U^\top U=I_k}\Tr(U^\top C U)
\]
is attained when the columns of \(U\) span the eigenspace associated with
the \(k\) smallest eigenvalues of \(C\). Equivalently, the retained
subspace is spanned by the \(d-k\) leading eigendirections of \(C\), which
is precisely the uncentered PCA solution.

\subsection{Derivation of Proposition 1}
\label{sec:prop_one_derivation}
\propone*
\begin{proof}
For any \(a,b\geq0\),
\[
    ab=\frac14\inf_{\tau>0}(\tau a+\tau^{-1}b)^2.
\]
Applying this identity with \(a=u^\top Cu\) and \(b=u^\top Hu\),
and interchanging the infima, gives
\begin{align*}
    \min_{\|u\|_2=1}(u^\top Cu)(u^\top Hu)
    &=
    \frac14\inf_{\tau>0}
    \min_{\|u\|_2=1}
    \bigl[u^\top(\tau C+\tau^{-1}H)u\bigr]^2\\
    &=
    \frac14\inf_{\tau>0}
    \bigl[\lambda_{\min}(\tau C+\tau^{-1}H)\bigr]^2,
\end{align*}
where the last equality follows from Rayleigh--Ritz and
\(\tau C+\tau^{-1}H\succeq0\).

If \(C,H\succ0\), the spectral objective is continuous and
diverges as \(\tau\to0\) or \(\tau\to\infty\), since
\[
    \lambda_{\min}(\tau C+\tau^{-1}H)
    \geq
    \tau\lambda_{\min}(C)+\tau^{-1}\lambda_{\min}(H).
\]
It therefore attains its minimum at some \(\tau_\star>0\).
For any corresponding unit bottom eigenvector \(u_\star\),
the scalar inequality above bounds its loss by the global minimum
in \eqref{eq:rankone_equivalence}; hence \(u_\star\) is globally optimal.
\end{proof}

\subsection{Derivation of Proposition 2}
\proptwo*

\begin{proof}
Set \(A=U^\top CU\) and \(D=U^\top HU\). Since \(A,D\) are
symmetric,
\begin{align*}
    4\mathcal{L}(U)
    &=
    \|\tau A+\tau^{-1}D\|_F^2
    -
    \|\tau A-\tau^{-1}D\|_F^2\\
    &\leq
    \|U^\top M_\tau U\|_F^2 \leq
    \|M_\tau U\|_F^2=
    \Tr(U^\top M_\tau^2U).
\end{align*}
The second inequality follows from \(U^\top U=I_k\), since
multiplication by \(U^\top\) cannot increase the Frobenius norm.
Dividing by four proves \eqref{eq:upper_bound}.

By the Rayleigh--Ritz theorem,
\[
    \min_{U^\top U=I_k}\mathcal{B}_\tau(U)
    =
    \frac14\sum_{j=1}^k\lambda_j(M_\tau^2),
\]
with a minimizer formed from the corresponding eigenvectors.
Since \(M_\tau\succeq0\), squaring preserves its eigenvectors
and eigenvalue ordering, so
\(\lambda_j(M_\tau^2)=\lambda_j(M_\tau)^2\).
This proves the stated spectral characterization.
\end{proof}

\subsection{Derivation of Proposition 3}
\propthree*
\begin{proof}
For \(U^\top U=I_k\), set \(A=U^\top CU\), \(D=U^\top HU\), and
define
\[
    \mathcal{S}_\tau(U)
    =
    \frac14\|U^\top M_\tau U\|_F^2.
\]
By eigenvalue interlacing,
\(\lambda_j(U^\top M_\tau U)\geq\lambda_j(M_\tau)\geq0\).
Consequently,
\[
    \min_{U^\top U=I_k}\mathcal{S}_\tau(U)
    =
    \frac14\sum_{j=1}^k\lambda_j(M_\tau)^2,
\]
with equality attained by a bottom eigenspace of \(M_\tau\).
By Proposition~\ref{prop_two}, this is also the minimum of
\(\mathcal{B}_\tau(U)\). For fixed \(U\), expanding the squared norm gives
\[
    \mathcal{S}_\tau(U)
    =
    \frac14\left(
        \tau^2\|A\|_F^2
        +\tau^{-2}\|D\|_F^2
        +2\mathcal{L}(U)
    \right).
\]
The scalar identity
\(\inf_{t>0}(ta+t^{-1}b)=2\sqrt{ab}\), for \(a,b\geq0\),
therefore yields
\[
    \inf_{\tau>0}\mathcal{S}_\tau(U)
    =
    \frac12\mathcal{L}(U)
    +
    \frac12\|A\|_F\|D\|_F.
\]
Exchanging the infima establishes
\begin{align*}
    \frac14\inf_{\tau>0}
    \sum_{j=1}^k\lambda_j(M_\tau)^2
    &=
    \inf_{\tau>0}\min_{U^\top U=I_k}\mathcal{S}_\tau(U)\\
    &=
    \min_{U^\top U=I_k}
    \left[
        \frac12\mathcal{L}(U)
        +\frac12\|U^\top CU\|_F\|U^\top HU\|_F
    \right].
\end{align*}
The final minimum exists because its objective is continuous
on the compact feasible set.

If \(C,H\succ0\), the spectral objective is continuous and
diverges at both endpoints \(\tau\to0\) and \(\tau\to\infty\),
so its infimum is attained. Moreover, for any minimizer \(U\)
of the regularized objective, choosing
\[
    \tau^2=\frac{\|U^\top HU\|_F}{\|U^\top CU\|_F}
\]
gives a joint minimizer of \(\mathcal{S}_\tau(U)\).
Such a subspace is a bottom eigenspace of \(M_\tau\), where
\(\mathcal{B}_\tau(U)=\mathcal{S}_\tau(U)\).
Conversely, any joint minimizer of \(\mathcal{B}_\tau(U)\)
minimizes the regularized objective, since
\(\mathcal{S}_\tau(U)\leq\mathcal{B}_\tau(U)\) and their joint
optimal values agree.
Finally, removing the common positive factor \(1/2\) leaves
the minimizing subspaces unchanged.
\end{proof}

\section{Orthogonal Reparameterization Invariance of LLMs}
\label{sec:orthog_inv}

We briefly review the computational invariance described by
\citet{slicegpt}, which arises from an orthogonal equivariance of the
hidden representations under a corresponding reparameterization of the
network weights. Let \(x\in\mathbb{R}^{d}\) denote a hidden
representation and consider a transformer block
\[
    F(x)
    =
    W_{\mathrm{out}}\,
    \sigma\!\left(W_{\mathrm{in}}x\right),
\]
where \(W_{\mathrm{in}}\in\mathbb{R}^{m\times d}\) and
\(W_{\mathrm{out}}\in\mathbb{R}^{d\times m}\) are the input and output
projections, respectively, and
\(\sigma:\mathbb{R}^{m}\to\mathbb{R}^{m}\) denotes the intervening
nonlinear operation. Let \(Q\in\mathbb{R}^{d\times d}\) be orthogonal,
so that \(Q^{\top}Q=QQ^{\top}=I_d\), and define
\[
    \widetilde{x}=Qx,\qquad
    \widetilde{W}_{\mathrm{in}}=W_{\mathrm{in}}Q^{\top},
    \qquad
    \widetilde{W}_{\mathrm{out}}=QW_{\mathrm{out}}.
\]
The transformed block then satisfies
\[
    \widetilde{W}_{\mathrm{out}}\,
    \sigma\!\left(
        \widetilde{W}_{\mathrm{in}}\widetilde{x}
    \right)
    =
    QW_{\mathrm{out}}\,
    \sigma\!\left(W_{\mathrm{in}}x\right)
    =
    QF(x).
\]
Thus, \(Q\) need not commute with \(\sigma\): it is canceled before the
nonlinearity and restored by the output projection. The residual update
is preserved in the rotated basis because
\[
    \widetilde{x}+\widetilde{F}(\widetilde{x})
    =
    Qx+QF(x)
    =
    Q\bigl(x+F(x)\bigr).
\]
Finally, the unscaled RMSNorm operator
\[
    \mathcal{N}(x)
    =
    \frac{x}{
        \sqrt{d^{-1}\lVert x\rVert_2^2+\varepsilon}
    },
    \qquad \varepsilon\geq 0,
\]
is orthogonally equivariant:
\[
    \mathcal{N}(Qx)=Q\mathcal{N}(x).
\]
Moreover, the learned coordinate-wise scale parameters of RMSNorm can be
absorbed into the input projection \(W_{\mathrm{in}}\) of the subsequent
block. LayerNorm, in contrast, is not equivariant under arbitrary
orthogonal transformations because its mean-centering operation
distinguishes the all-ones direction. SliceGPT addresses this by
rewriting LayerNorm-connected transformers in an equivalent
RMSNorm-connected form, absorbing the centering and affine operations
into adjacent linear maps. The change of basis can then be propagated
across transformer blocks without requiring the intervening nonlinear
operations to be orthogonally equivariant.

\section{Implementation details}
\label{sec:implementation}

We follow SliceGPT and construct pruning bases sequentially. After pruning each site, activations are propagated through the partially pruned model before computing the basis for the next site. For our method, we search over \(\tau \in \{1,7,10,30,70\}\). We estimate the sensitivity matrices \(H\) using one token sampled from the model output distribution per token and accumulate these estimates in single precision. Computing an independent backward pass for every sampled class would be prohibitively expensive; instead, we average the corresponding sampled log-probabilities and obtain the corresponding gradient using a single backward pass.

Following SliceGPT, we accumulate the activation covariance matrices \(C\) in double precision, and all eigendecompositions are likewise performed in double precision. For our method, we frobenius normalize both \(C\) and \(H\) prior to performing the sweep over \(\tau\). 

We evaluate common-sense reasoning and mathematical reasoning/instruction following using the Language Model Evaluation Harness \citep{lmeval}. Common-sense reasoning tasks are evaluated zero-shot and normalized accuracy is used for all tasks when available. GSM8K is evaluated 8-shot with chain-of-thought \citep{chain_of_thought} prompting, and IFEval zero-shot. Both use each model's chat template. For GSM8K, we report flexible extraction results for Mistral Nemo as strict extraction mistakenly labels many generations as incorrect for the uncompressed model.

\section{Per-task common-sense reasoning results}
\label{sec:per_task_results}

In this section, we report the complete common-sense reasoning results used to produce the average accuracies shown in Table~\ref{tab:main_results}. Tables~\ref{tab:per_task_llama32_3b}--\ref{tab:per_task_phi3_medium} show these results per model.

\begin{table}[H]
\centering
\small
\setlength{\tabcolsep}{4pt}
\renewcommand{\arraystretch}{1.15}
\caption{
Per-task common-sense reasoning accuracy on Llama-3.2 3B Instruct.
}
\label{tab:per_task_llama32_3b}
\begin{tabular}{llcccccccc}
\toprule
\textbf{Sparsity} & \textbf{Method} & \textbf{WG} & \textbf{ARC-E} & \textbf{ARC-C} & \textbf{HS} & \textbf{OBQA} & \textbf{PIQA} & \textbf{MMLU} & \textbf{Avg.} \\
\midrule
0\% & N/A & 67.32 & 67.85 & 46.16 & 70.46 & 36.00 & 75.52 & 60.53 & 60.55 \\
\midrule
\multirow{2}{*}{10\%}
& SliceGPT & 65.59 & 68.90 & 42.06 & 65.49 & 37.00 & 73.67 & 54.09 & 58.11 \\
& Ours & 67.17 & 67.98 & 42.58 & 65.54 & 37.40 & 73.61 & 55.42 & 58.53 \\
\midrule
\multirow{2}{*}{20\%}
& SliceGPT & 64.40 & 63.89 & 37.54 & 56.64 & 35.40 & 70.13 & 46.45 & 53.49 \\
& Ours & 65.27 & 64.39 & 40.10 & 58.61 & 35.60 & 70.89 & 49.55 & 54.92 \\
\midrule
\multirow{2}{*}{25\%}
& SliceGPT & 62.43 & 59.01 & 34.56 & 51.95 & 34.40 & 68.17 & 43.46 & 50.57 \\
& Ours & 63.22 & 61.49 & 37.20 & 53.86 & 33.20 & 68.61 & 46.08 & 51.95 \\
\midrule
\multirow{2}{*}{30\%}
& SliceGPT & 60.54 & 55.09 & 32.94 & 47.67 & 33.20 & 64.69 & 37.69 & 47.40 \\
& Ours & 61.80 & 58.46 & 33.45 & 49.35 & 32.60 & 66.81 & 43.56 & 49.43 \\
\bottomrule
\end{tabular}
\end{table}

\begin{table}[H]
\centering
\small
\setlength{\tabcolsep}{4pt}
\renewcommand{\arraystretch}{1.15}
\caption{
Per-task common-sense reasoning accuracy on Llama-3.1 8B Instruct.
}
\label{tab:per_task_llama31_8b}
\begin{tabular}{llcccccccc}
\toprule
\textbf{Sparsity} & \textbf{Method} & \textbf{WG} & \textbf{ARC-E} & \textbf{ARC-C} & \textbf{HS} & \textbf{OBQA} & \textbf{PIQA} & \textbf{MMLU} & \textbf{Avg.} \\
\midrule
0\% & N/A & 73.88 & 79.59 & 54.95 & 79.25 & 43.00 & 80.96 & 68.10 & 68.53 \\
\midrule
\multirow{2}{*}{10\%}
& SliceGPT & 73.16 & 77.90 & 51.96 & 74.66 & 42.60 & 78.67 & 61.52 & 65.78 \\
& Ours & 72.77 & 78.32 & 52.82 & 75.27 & 42.00 & 79.00 & 63.70 & 66.27 \\
\midrule
\multirow{2}{*}{20\%}
& SliceGPT & 68.75 & 73.48 & 46.16 & 66.71 & 41.80 & 74.37 & 56.16 & 61.06 \\
& Ours & 70.56 & 75.51 & 46.93 & 67.15 & 41.80 & 74.97 & 58.00 & 62.13 \\
\midrule
\multirow{2}{*}{25\%}
& SliceGPT & 67.01 & 69.44 & 42.49 & 61.60 & 39.60 & 71.93 & 53.75 & 57.97 \\
& Ours & 67.40 & 71.93 & 44.88 & 62.16 & 39.40 & 73.04 & 54.33 & 59.02 \\
\midrule
\multirow{2}{*}{30\%}
& SliceGPT & 64.33 & 63.93 & 37.88 & 55.08 & 38.40 & 69.04 & 49.30 & 53.99 \\
& Ours & 63.77 & 68.73 & 41.89 & 56.53 & 37.60 & 71.00 & 50.49 & 55.72 \\
\bottomrule
\end{tabular}
\end{table}

\begin{table}[H]
\centering
\small
\setlength{\tabcolsep}{4pt}
\renewcommand{\arraystretch}{1.15}
\caption{
Per-task common-sense reasoning accuracy on Mistral 7B v0.3 Instruct.
}
\label{tab:per_task_mistral_7b}
\begin{tabular}{llcccccccc}
\toprule
\textbf{Sparsity} & \textbf{Method} & \textbf{WG} & \textbf{ARC-E} & \textbf{ARC-C} & \textbf{HS} & \textbf{OBQA} & \textbf{PIQA} & \textbf{MMLU} & \textbf{Avg.} \\
\midrule
0\% & N/A & 74.11 & 82.66 & 58.87 & 82.90 & 47.20 & 82.64 & 59.69 & 69.72 \\
\midrule
\multirow{2}{*}{10\%}
& SliceGPT & 74.43 & 81.31 & 55.80 & 78.93 & 46.40 & 80.36 & 58.17 & 67.91 \\
& Ours & 74.51 & 81.23 & 56.23 & 78.80 & 46.00 & 81.18 & 58.18 & 68.02 \\
\midrule
\multirow{2}{*}{20\%}
& SliceGPT & 71.82 & 78.37 & 51.88 & 71.26 & 45.20 & 76.17 & 55.05 & 64.25 \\
& Ours & 72.69 & 80.98 & 52.65 & 71.32 & 46.00 & 76.39 & 55.11 & 65.02 \\
\midrule
\multirow{2}{*}{25\%}
& SliceGPT & 69.46 & 76.18 & 48.98 & 65.63 & 43.40 & 74.10 & 50.87 & 61.23 \\
& Ours & 69.69 & 78.70 & 51.11 & 65.90 & 45.40 & 74.43 & 52.41 & 62.52 \\
\midrule
\multirow{2}{*}{30\%}
& SliceGPT & 67.32 & 72.81 & 45.14 & 59.16 & 40.40 & 70.08 & 46.90 & 57.40 \\
& Ours & 67.09 & 75.80 & 49.66 & 59.87 & 43.00 & 71.27 & 48.42 & 59.30 \\
\bottomrule
\end{tabular}
\end{table}

\begin{table}[H]
\centering
\small
\setlength{\tabcolsep}{4pt}
\renewcommand{\arraystretch}{1.15}
\caption{
Per-task common-sense reasoning accuracy on Mistral Nemo Instruct.
}
\label{tab:per_task_mistral_nemo}
\begin{tabular}{llcccccccc}
\toprule
\textbf{Sparsity} & \textbf{Method} & \textbf{WG} & \textbf{ARC-E} & \textbf{ARC-C} & \textbf{HS} & \textbf{OBQA} & \textbf{PIQA} & \textbf{MMLU} & \textbf{Avg.} \\
\midrule
0\% & N/A & 74.90 & 79.97 & 58.87 & 82.42 & 46.40 & 82.21 & 65.60 & 70.05 \\
\midrule
\multirow{2}{*}{10\%}
& SliceGPT & 71.03 & 77.19 & 56.83 & 77.08 & 43.80 & 80.69 & 62.78 & 67.06 \\
& Ours & 74.19 & 77.02 & 56.66 & 77.01 & 45.40 & 80.85 & 63.22 & 67.76 \\
\midrule
\multirow{2}{*}{20\%}
& SliceGPT & 69.61 & 74.49 & 50.68 & 67.37 & 38.60 & 76.17 & 59.04 & 62.28 \\
& Ours & 70.72 & 75.21 & 51.37 & 67.50 & 40.80 & 75.79 & 59.04 & 62.92 \\
\midrule
\multirow{2}{*}{25\%}
& SliceGPT & 66.89 & 71.21 & 46.25 & 61.65 & 37.60 & 73.01 & 55.82 & 58.92 \\
& Ours & 68.03 & 73.48 & 47.87 & 61.32 & 39.60 & 72.96 & 56.90 & 60.02 \\
\midrule
\multirow{2}{*}{30\%}
& SliceGPT & 63.54 & 65.24 & 40.19 & 53.94 & 36.20 & 68.72 & 50.38 & 54.03 \\
& Ours & 64.40 & 69.32 & 43.52 & 54.78 & 37.20 & 69.86 & 51.53 & 55.80 \\
\bottomrule
\end{tabular}
\end{table}

\begin{table}[H]
\centering
\small
\setlength{\tabcolsep}{4pt}
\renewcommand{\arraystretch}{1.15}
\caption{
Per-task common-sense reasoning accuracy on Phi-3 Medium Instruct.
}
\label{tab:per_task_phi3_medium}
\begin{tabular}{llcccccccc}
\toprule
\textbf{Sparsity} & \textbf{Method} & \textbf{WG} & \textbf{ARC-E} & \textbf{ARC-C} & \textbf{HS} & \textbf{OBQA} & \textbf{PIQA} & \textbf{MMLU} & \textbf{Avg.} \\
\midrule
0\% & N/A & 76.56 & 81.36 & 61.60 & 82.76 & 50.60 & 81.66 & 76.14 & 72.95 \\
\midrule
\multirow{2}{*}{10\%}
& SliceGPT & 76.56 & 85.98 & 63.74 & 80.37 & 49.60 & 81.83 & 71.79 & 72.84 \\
& Ours & 77.43 & 84.34 & 61.77 & 80.38 & 48.60 & 81.18 & 74.21 & 72.56 \\
\midrule
\multirow{2}{*}{20\%}
& SliceGPT & 73.24 & 82.66 & 58.36 & 74.26 & 46.00 & 79.05 & 66.81 & 68.63 \\
& Ours & 74.98 & 82.62 & 58.02 & 75.46 & 46.20 & 79.71 & 69.80 & 69.54 \\
\midrule
\multirow{2}{*}{25\%}
& SliceGPT & 71.74 & 78.83 & 54.69 & 70.33 & 44.00 & 77.04 & 59.58 & 65.17 \\
& Ours & 75.77 & 81.19 & 55.72 & 71.86 & 44.80 & 77.58 & 66.14 & 67.58 \\
\midrule
\multirow{2}{*}{30\%}
& SliceGPT & 69.77 & 74.96 & 50.43 & 66.43 & 42.40 & 74.86 & 53.63 & 61.78 \\
& Ours & 73.32 & 78.28 & 53.24 & 66.65 & 44.00 & 76.06 & 62.77 & 64.90 \\
\bottomrule
\end{tabular}
\end{table}

\section{Results with RedPajama calibration}

The results presented use FineWeb-Edu for language modeling and common-sense reasoning evaluation results. In this section, we explore whether our method continues to improve performance when using a different calibration set. Specifically, we generate results on Llama-3.2 3B Instruct and Llama-3.1 8B Instruct using 1,024 samples of length 2,048 from the RedPajama \citep{redpajama} dataset and show results in Table~\ref{tab:redpajama_results}.

\begin{table}[h]
\centering
\small
\setlength{\tabcolsep}{4pt}
\renewcommand{\arraystretch}{1.15}
\caption{
Results using RedPajama as the calibration set.
We report KL divergence, perplexity (PPL), and average downstream accuracy (Acc.).
}
\label{tab:redpajama_results}
\begin{tabular}{llcccccc}
\toprule
& &
\multicolumn{3}{c}{\textbf{Llama-3.2 3B Instruct}} &
\multicolumn{3}{c}{\textbf{Llama-3.1 8B Instruct}} \\
\cmidrule(lr){3-5}
\cmidrule(lr){6-8}
\textbf{Sparsity}
& \textbf{Method}
& \(\mathbf{\KL}\) & \textbf{PPL} & \textbf{Acc.}
& \(\mathbf{\KL}\) & \textbf{PPL} & \textbf{Acc.} \\
\midrule
0\%
& N/A
& 0.0 & 11.76 & 60.55
& 0.0 & 7.21 & 68.53 \\
\midrule

\multirow{2}{*}{10\%}
& SliceGPT
& 0.2985 & 16.43 & 56.81
& 0.5505 & 12.60 & 64.60 \\
& Ours
& 0.1540 & 13.95 & 57.25
& 0.4458 & 11.49 & 65.18 \\
\midrule

\multirow{2}{*}{20\%}
& SliceGPT
& 0.6387 & 23.04 & 50.55
& 1.0121 & 20.31 & 57.25 \\
& Ours
& 0.4209 & 18.31 & 51.44
& 0.8330 & 17.09 & 57.93 \\
\midrule

\multirow{2}{*}{25\%}
& SliceGPT
& 0.8899 & 29.58 & 46.61
& 1.2805 & 26.59 & 51.91 \\
& Ours
& 0.6143 & 22.48 & 47.91
& 1.0887 & 22.10 & 52.83 \\
\midrule

\multirow{2}{*}{30\%}
& SliceGPT
& 1.1147 & 36.77 & 43.23
& 1.5838 & 35.99 & 46.64 \\
& Ours
& 0.8363 & 29.58 & 44.01
& 1.3503 & 28.64 & 47.48 \\
\bottomrule
\end{tabular}
\end{table}

\section{Results on non-instruction tuned models}
\label{sec:base_results}

Our main experiments focus on instruction-tuned models. To evaluate whether our method remains effective for models without instruction tuning, we additionally prune the base Llama-2 7B \citep{llama_2}, Llama-3.2 3B, and Llama-3.1 8B models and report results in Table~\ref{tab:base_results}. Our method achieves lower KL divergence and perplexity than SliceGPT across every model and sparsity level, and higher average common-sense reasoning accuracy in all settings.

\begin{table}[h]
\centering
\footnotesize
\setlength{\tabcolsep}{3pt}
\renewcommand{\arraystretch}{1.15}
\caption{
Results on non-instruction-tuned models.
We report KL divergence, perplexity (PPL), and average downstream accuracy (Acc.).
}
\label{tab:base_results}
\begin{tabular}{llccccccccc}
\toprule
& &
\multicolumn{3}{c}{\textbf{Llama-2 7B}} &
\multicolumn{3}{c}{\textbf{Llama-3.2 3B}} &
\multicolumn{3}{c}{\textbf{Llama-3.1 8B}} \\
\cmidrule(lr){3-5}
\cmidrule(lr){6-8}
\cmidrule(lr){9-11}
\textbf{Sparsity}
& \textbf{Method}
& \(\mathbf{\KL}\) & \textbf{PPL} & \textbf{Acc.}
& \(\mathbf{\KL}\) & \textbf{PPL} & \textbf{Acc.}
& \(\mathbf{\KL}\) & \textbf{PPL} & \textbf{Acc.} \\
\midrule
0\%
& N/A
& 0.0 & 5.47  & 61.56 
& 0.0 & 7.82 & 62.19
& 0.0 & 6.24 & 68.03 \\
\midrule

\multirow{2}{*}{10\%}
& SliceGPT
& 0.3736 & 8.16 & 60.01
& 0.1770 & 15.59 & 57.44
& 0.9215 & 15.92 & 64.75 \\
& Ours
& 0.2007 & 6.86 & 60.10
& 0.1340 & 12.56 & 58.31
& 0.5836 & 11.42 & 65.53 \\
\midrule

\multirow{2}{*}{20\%}
& SliceGPT
& 0.9157 & 14.68 & 56.57
& 0.4287 & 30.46 & 51.78
& 1.4205 & 26.42 & 58.99 \\
& Ours
& 0.5178 & 9.62 & 56.82
& 0.3377 & 19.59 & 53.25
& 1.1735 & 20.82 & 60.42 \\
\midrule

\multirow{2}{*}{25\%}
& SliceGPT
& 1.1383 & 18.49 & 53.79
& 0.5927 & 40.73 & 48.23
& 1.7261 & 35.85 & 55.18 \\
& Ours
& 0.7199 & 11.88 & 54.90
& 0.4675 & 24.87 & 50.82
& 1.4607 & 27.75 & 55.82 \\
\midrule

\multirow{2}{*}{30\%}
& SliceGPT
& 1.3793 & 23.77 & 51.29
& 0.7746 & 57.20 & 44.02
& 2.1285 & 53.59 & 50.06 \\
& Ours
& 0.9473 & 15.02 & 51.78
& 0.6173 & 32.48 & 47.07
& 1.7200 & 35.84 & 52.13 \\
\bottomrule
\end{tabular}
\end{table}

\section{Analysis of Dataset-average Curvature}
\label{sec:curv_analysis}

\begin{figure*}[t]
\centering
\includegraphics[width=\linewidth]{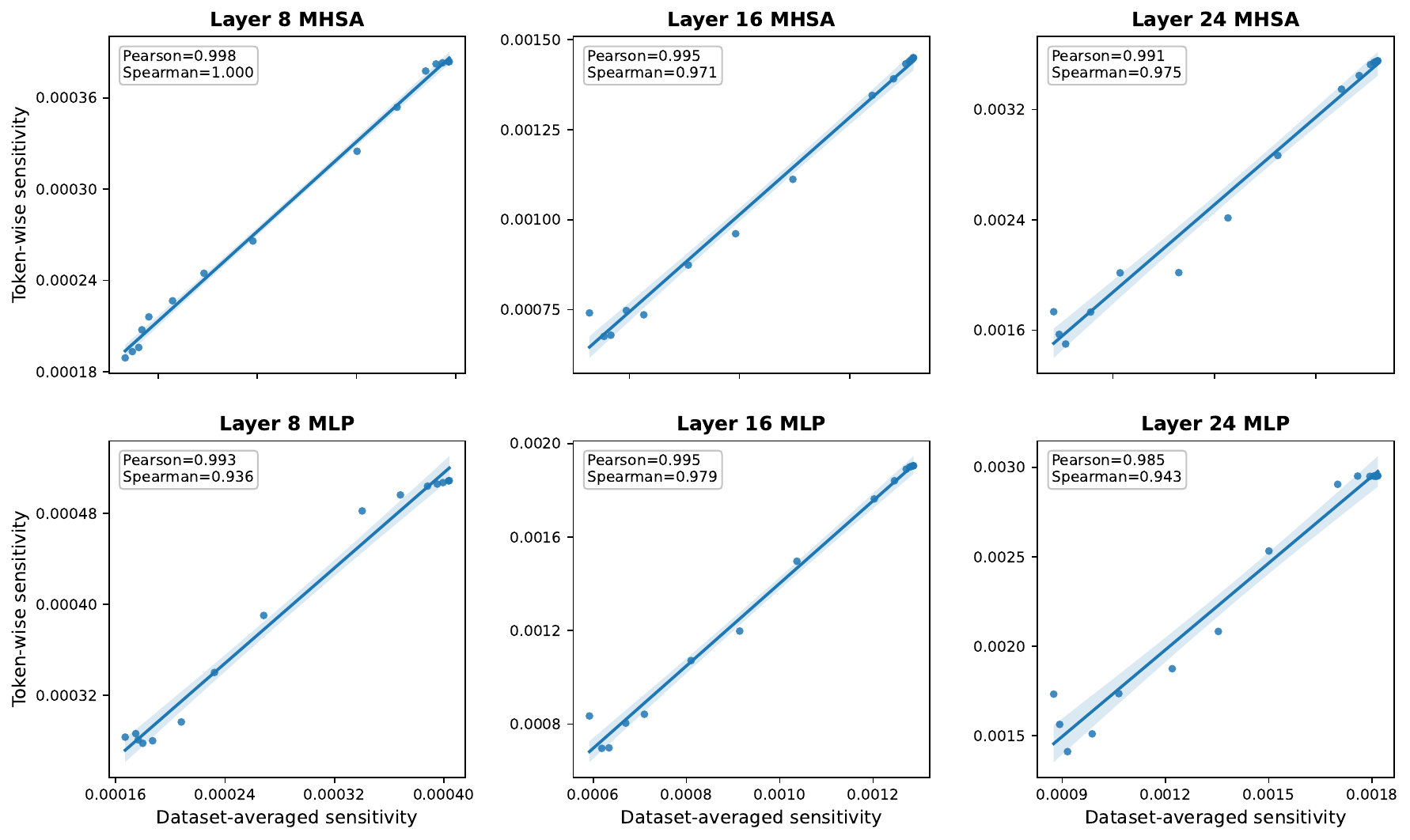}
\caption{
Comparison of dataset-averaged and token-wise sensitivity across candidate pruning bases.
For each basis, we rescale its induced perturbations by a single scalar such that all bases have equal total reconstruction error, isolating differences in output sensitivity from differences in perturbation magnitude.
Each panel reports the correlation between the objective computed using the dataset-averaged curvature matrix \(H\) and the corresponding objective computed using token-specific matrices \(H_i(s)\).
The strong agreement indicates that averaging sensitivity across the calibration distribution preserves the relative sensitivity of pruning-relevant directions.
}
\label{fig:dataset_vs_token_sens}
\end{figure*}

Our objective replaces the token-dependent curvature matrices \(H_i(s)\) with a single dataset-averaged matrix \(H\). While this makes basis generation computationally tractable, it removes the explicit dependence between each token activation and its corresponding output sensitivity. In particular, if the directions to which the model output is sensitive vary substantially across tokens, a shared matrix \(H\) may fail to accurately characterize the sensitivity of the perturbations induced by pruning.

We empirically evaluate the effect of this approximation by sampling \(K=128\) tokens from the calibration set and estimating the corresponding token-specific curvature matrix \(H_i(s)\) for each token. Because individual estimates are substantially noisier than the dataset-level estimate used by our method, we use \(64\) Monte Carlo samples from the output distribution for each \(H_i(s)\).

For a candidate pruning basis \(U\), let

$$
r_i(s;U) = UU^\top x_i(s)
$$

denote the corresponding reconstruction residual. We compare the sensitivity assigned to these residuals by the dataset-averaged metric,

$$
\mathcal{L}_{\mathrm{avg}}(U)
=
\frac{1}{K}
\sum_{s,i}
r_i(s;U)^\top H\,r_i(s;U),
$$

with the sensitivity obtained using the corresponding token-specific matrices,

$$
\mathcal{L}_{\mathrm{tok}}(U)
=
\frac{1}{K}
\sum_{s,i}
r_i(s;U)^\top H_i(s)\,r_i(s;U).
$$

Directly comparing these quantities across different bases can confound sensitivity with reconstruction error: a basis that produces larger residuals will generally receive a larger value under both metrics regardless of whether \(H\) accurately approximates the token-specific sensitivity. We therefore rescale the residual induced by each basis using a single scalar \(\gamma_U\), chosen such that all candidate bases produce the same total reconstruction error. Defining

$$
\widetilde r_i(s;U)
=
\gamma_U r_i(s;U),
$$

we evaluate the reconstruction-error-matched objectives

$$
\widetilde{\mathcal{L}}_{\mathrm{avg}}(U)
=
\frac{1}{K}
\sum_{s,i}
\widetilde r_i(s;U)^\top
H
\widetilde r_i(s;U)
$$

and

$$
\widetilde{\mathcal{L}}_{\mathrm{tok}}(U)
=
\frac{1}{K}
\sum_{s,i}
\widetilde r_i(s;U)^\top
H_i(s)
\widetilde r_i(s;U).
$$

Figure~\ref{fig:dataset_vs_token_sens} compares these two objectives across candidate pruning bases at several layers. Despite replacing the individual \(H_i(s)\) with a shared dataset-level metric, we observe strong correlation between \(\widetilde{\mathcal{L}}_{\mathrm{avg}}\) and \(\widetilde{\mathcal{L}}_{\mathrm{tok}}\). Importantly, this agreement persists after matching reconstruction error across bases, indicating that the correlation is not explained solely by differences in perturbation magnitude. Instead, the dataset-averaged curvature matrix largely preserves the relative sensitivity assigned to pruning-relevant perturbation directions. Together with the correlation between our surrogate objective and the resulting output KL divergence shown in Figure \ref{fig:surrogate_kl}, these results support the use of a shared dataset-averaged sensitivity matrix for basis selection.

\end{document}